\documentclass[conference]{IEEEtran}
\IEEEoverridecommandlockouts
\usepackage{cite}
\usepackage{amsmath,amssymb,amsfonts}
\usepackage{algorithmic}
\usepackage{graphicx}
\usepackage{textcomp}
\usepackage{xcolor}
\usepackage{units}
\usepackage{multirow}
\usepackage{colortbl}

\usepackage{amsthm}
\usepackage{thmtools, thm-restate}
\usepackage{algorithm}

\renewcommand{\c}{\mathbf{c}}

\newcommand{\e}{\mathbf{e}}

\newcommand{\MH}{\mathbf{H}}

\newcommand{\K}{\mathbf{K}}

\newcommand{\w}{\mathbf{w}}

\newcommand{\x}{\mathbf{x}}

\newcommand{\y}{\mathbf{y}}
\newcommand{\z}{\mathbf{z}}

\newcommand{\0}{\mathbf{0}}

\newcommand{\mc}[1]{\mathcal{#1}}
\newcommand{\mb}[1]{\mathbb{#1}}
\newcommand{\mf}[1]{\mathbf{#1}}

\newtheorem{ex}{Example}

\newtheorem{assum}{Assumption}

\newcommand{\alg}{{ACPC-OTA-FL}}

\def\BibTeX{{\rm B\kern-.05em{\sc i\kern-.025em b}\kern-.08em
    T\kern-.1667em\lower.7ex\hbox{E}\kern-.125emX}}
\begin{document}

\title{Over-the-Air Federated Learning with Joint Adaptive Computation and Power Control
\thanks{This work is supported in part by NSF-2112471 (AI-EDGE).}
}


\author{\IEEEauthorblockN{1\textsuperscript{st} Given Name Surname}
\IEEEauthorblockA{\textit{dept. name of organization (of Aff.)} \\
\textit{name of organization (of Aff.)}\\
City, Country \\
email address or ORCID}
}

\author{Haibo Yang, Peiwen Qiu, Jia Liu and Aylin Yener
\\Dept. of Electrical and Computer Engineering, The Ohio State University
\\ \{yang.9292, qiu.617\}@osu.edu, liu@ece.osu.edu, yener@ece.osu.edu
}

\maketitle


\begin{abstract}
This paper considers over-the-air federated learning (OTA-FL). OTA-FL exploits the superposition property of the wireless medium, and performs model aggregation over the air for free. Thus, it can greatly reduce the communication cost incurred in communicating model updates from the edge devices. In order to fully utilize this advantage while providing comparable learning performance to conventional federated learning that presumes model aggregation via noiseless channels, we consider the joint design of transmission scaling and the number of local iterations at each round, given the power constraint at each edge device. We first characterize the training error due to such channel noise in OTA-FL by establishing a fundamental lower bound for general functions with Lipschitz-continuous gradients. Then, by introducing an adaptive transceiver power scaling scheme, we propose an over-the-air federated learning algorithm with joint adaptive computation and power control (ACPC-OTA-FL). We provide the convergence analysis for ACPC-OTA-FL in training with non-convex objective functions and heterogeneous data. We show that the convergence rate of ACPC-OTA-FL matches that of FL with noise-free communications.
\end{abstract}



\section{Introduction} \label{sec:intro}

In recent years, advances in machine learning (ML) have achieved astonishing successes in many applications that transform our society, e.g., in computer vision, natural language processing, and robotics.
Traditionally, ML training tasks often reside in cloud-based large data-centers that process training data in a centralized fashion. 
However, due to the rapidly increasing demands for training data, high latency and costs of data transmissions, as well as data privacy/security concerns, aggregating all data to the cloud for ML training is unlikely to remain feasible. To address these challenges, federated learning (FL) \cite{mcmahan2017communication} has recently
emerged as a prevailing distributed ML paradigm. 
FL employs multiple clients, typically deployed over wireless edge networks, to locally train a learning model and exchange only intermediate updates between the server and clients.
FL provides better avenues for privacy protection by avoiding the transmission of local data, while also being able to leverage parallel clients computation for training speedup.

However, FL also inherits many design challenges of distributed ML.
One of the main challenges of FL stems from the communication constraint in the iterative FL learning process, particularly in resource (bandwidth and power)-limited wireless FL systems~\cite{yang2019flconcept,mcmahan2021fl,niknam2020federated}.
To receive the update information from multiple clients in each round, the conventional wisdom is to use orthogonal spectral or temporal channels for each client and avoid interference among the clients.
However, this is neither desirable (since as the number of clients increases, the available rate per edge device decreases, lengthening the communication duration), nor necessary (since only the aggregated model updates is needed at the server) in FL. 

Over-the-air FL (OTA-FL) has recently emerged as an effective approach in that it exploits the superposition property of the wireless medium to perform model aggregation ``for free'' by allowing simultaneous transmission of all clients' updates \cite{abari2016over,yang2020federated,zhu2021over}.
Specifically, under OTA-FL, the server directly recovers a noisy aggregation of the clients' model updates that transmit in the same spectral-temporal channel, rather than trying to decode each client's model update first in orthogonal spectral or temporal channels. 
As a result, OTA-FL dramatically reduces the communication costs and overheads from collecting the update from each client, and accordingly enjoys better {\em communication parallelism} regardless of the number of clients.

However, despite of the aforementioned salient features in terms of communication efficiency, several issues remain in OTA-FL.
First, the convergence analysis of OTA-FL often assumes noise-free communications (see Section~\ref{sec:related} for more in-depth discussions).
Moreover, existing works on OTA-FL have not considered {\it data heterogeneity}, (i.e., datasets among clients are non-i.i.d. and with unbalanced sizes) and {\it system heterogeneity} (i.e., the computation and communication capacities varies among clients and could be time-varying~\cite{yang2019flconcept,mcmahan2021fl} simultaneously).
Therefore, a fundamental question in OTA-FL system design is: {\it how to develop an efficient OTA-FL training algorithm that can handle both data and system heterogeneity under noisy channels.}


In this paper, we answer this question by studying the impact of channel noise on OTA-FL and proposing an over-the-air federated learning algorithm with joint adaptive computation and power control (ACPC-OTA-FL) for edge devices with heterogeneous capabilities.
Our main contributions are summarized as follows:

\begin{list}{\labelitemi}{\leftmargin=1em \itemindent=0em \itemsep=.2em}

\item We first characterize the training error of the conventional OTA-based FedAvg algorithm\cite{mcmahan2017communication} by establishing a lower bound of the convergence error under Gaussian multiple access channels (MAC) for general functions with Lipschitz continuous gradients.
Our lower bound indicates that there is a non-vanishing convergence error due to channel noise compared with those convergence results of OTA-FL under the noise-free assumption.
This insight motivates us to propose our {\alg} algorithm that considers local computation and power control co-design to best utilize the power resources at the edge devices. 


\item Our {\alg} algorithm allows each client to (in a distributed manner) adaptively determine its {\bf transmission power level and number of local update steps}  to fully utilize the computation and communication resources.
We show that, even with non-i.i.d. and unbalanced datasets, {\alg} converges to a stationary point with an $\mathcal{O}(1/\sqrt{mT})$ convergence rate for nonconvex objective functions in training, where $m$ is the number of clients and $T$ is the total number of training iterations.
This result further implies a linear speedup in the number of clients and matches that of the noise-free FedAvg algorithm.
\end{list}


\section{Related Work} \label{sec:related}

OTA-FL utilizes over-the-air computation through analog transmission in wireless MAC~\cite{abari2016over,yang2019federated,yang2020federated,zhu2019broadband,sery2020analog,amiri2020federated,amiri2020machine,amiri2019over}.
Despite the advantage of high scalability for large amount of clients, existing works on OTA-FL\cite{zhu2019broadband,sery2020analog,amiri2020federated,amiri2020machine,amiri2019over} have empirically shown that the channel noise substantially degrades the learning performance. 
Therefore, theoretically quantifying the impact of channel noise on convergence needs more in-depth investigation (See Section~\ref{sec:channelnoise}).

To mitigate the impacts of channel noise under limited transmission power constraints, one popular approach is to utilize uniform transceiver scaling for all clients.
For example, references~\cite{amiri2020machine,amiri2020federated,zhu2020one} have considered an identical sparsity pattern to reduce communication overhead. Reference~\cite{guo2020analog} has proposed a new learning rate scheme that considers the quality of the gradient estimation; 
Reference~\cite{chen2018uniform} has developed a uniform-forcing transceiver scaling for OTA function computation, while the work in \cite{zhang2020gradient} has studied the optimal power control problem by taking gradient statistics into account. Reference~\cite{sery2021over} has proposed an uniform transceiver scaling by considering data heterogeneity.
A common approach of these existing works above is to formulate the power control problem separately to satisfy the power constraints after the local computation at clients. Moreover, data and system heterogeneity and adapting the power resources for computation is not tied to transmission power control, despite the fact that each edge device is constrained in total power that is needed for both computation and communication.
Due to the coupling of computation and communication processes, a joint adaptive computation and power control for OTA-FL is necessary in order to better mitigate the combined impacts of channel noise, power constraints, as well as data and system heterogeneity, which constitutes the main goal of this paper.

\section{System Model} \label{sec: prelim}

In this section, we first introduce our OTA-FL model in Section~\ref{subsec: fl} and then the communication model in Section~\ref{subsec: channel}.

\subsection{Over-the-Air Federated Learning Model} \label{subsec: fl}
We consider a FL system with one server and $m$ clients in total.
Each client $i \in [m]$ contains a private local dataset $D_i$.
Each local dataset $D_i$ is i.i.d. sampled from a distribution $\mc{X}_i$.
In this paper, we consider non-i.i.d. datasests across users, i.e., $\mc{X}_i \neq \mc{X}_j$ if $ i \neq j, \forall i, j \in [m]$.
The goal of FL is to collaboratively train a global learning model by solving an optimization problem in the form of:
\begin{equation}
    \min_{\x \in \mathbb{R}^d}F(\x) \triangleq \min_{\x\in\mathbb{R}^d} \sum_{i \in [m]} \alpha_i F_i(\x, D_i), 
    \label{eq: objective}
\end{equation}
where $\x \in \mb{R}^d$ is the model parameter,
$\alpha_i = \frac{| D_i |}{\sum_{i \in [m]} | D_i |}$ denotes the proportion of client $i$'s dataset size in the entire population.
In this paper, we consider the unbalanced data setting: $\alpha_i \ne \alpha_j$ if $i\ne j$.
In \eqref{eq: objective}, $F_i(\x, D_i) \triangleq \frac{1}{| D_i |} \sum_{\xi^i_j \in D_i} F(\x, \xi^i_j)$ is the local objective function, where $\xi^i_j$ is the $j$-th sample in $D_i$.
In FL, $F_i(\x, D_i)$ is assumed to be non-convex in general.

In each communication round $t$, the server broadcasts the latest global model parameter $\x_t$ to each client.
Upon receiving $\x_t$, client $i$ performs local computations with $\x^i_{t, 0} = \x_t$:
\begin{equation}
    \x^i_{t, k+1} = \x^i_{t, k} - \eta \nabla F(\x^i_{t, k}, \xi^i_{t, k}), \quad k = 0,\ldots,\tau_t^i-1,
\end{equation}
where $\tau_t^i$ denotes the total number of local steps at client $i$ in round $t$ and $\xi^i_{t, k}$ is a random data sample used by client $i$ in step $k$ in round $t$.
We note that one key feature of the OTA-FL model in this paper is that {\bf we allow $\tau_t^i$ to be time-varying and device-dependent}.
While this makes the OTA-FL model more practical and flexible, it also introduces an extra dimension of challenges in algorithmic design and convergence analysis.

After finishing the local iterations, each client returns the update of model parameters back to server.
Upon the reception of returned updates from all clients, the server aggregates and updates the global model $\x_t$.
In OTA-FL, the aggregation process on the server side happens automatically over-the-air thanks to the superposition property of the wireless medium.
The specific communication model will be described in Section~\ref{subsec: channel}.


\subsection{Communication Model} \label{subsec: channel}
We consider an OTA-FL system in which the server broadcasts to all clients in a downlink channel and the clients transmit to the server through a common uplink channel synchronously. We assume an error-free downlink when broadcasting the global model. This is a reasonable assumption when the server has access to more power and bandwidth resources compared to edge devices.
As a result, each client receives an error-free global model parameter $\x_t$ for its local computation in the beginning of each round $t$, i.e., $\x^i_{t, 0} = \x_t$.
For the uplink, we consider a Gaussian MAC, where the output the channel in each communication round $t$ is:
\begin{equation} \label{eqn:mac}
    \y_t = \sum_{i \in [m]} h_t^i \z^i_t + \w_t.
\end{equation}
In \eqref{eqn:mac}, $\z^i_t \in \mathbb{R}^d$ is the input from client $i$, $h_t^i$ is the channel gain of client $i$, and $\w_t \sim \mc{N}(\0, \sigma_c^2 \mf{I}_d)$ is an additive Gaussian channel noise.


We also consider an instantaneous power constraint at each client in each communication round:
\begin{equation}
    \| \z^i_t \|^2 \leq P_t^i, \quad \forall i \in [m], \forall t,
\end{equation}
where $P_t^i$ is the maximum transmission power limit for client $i$ in communication round $t$.

\section{Impacts of Channel Noise and System-Data Heterogeneity on OTA-FL} \label{sec:channelnoise}


In Section~\ref{subsec:noise}, we first characterize the impact of the channel noise on OTA-FL when directly applying the standard FedAvg framework with SGD local updates without considering power control at each client.
Then, in Section~\ref{subsec:heterogeneity}, we provide a concrete example to further illustrate the impact of channel noise coupled with heterogeneous numbers of local updates, i.e., system heterogeneity, on OTA-FL performance. 

\subsection{Impact of Channel Noise on OTA-FL} \label{subsec:noise}
To study the impact of channel noise, we first consider a general $L$-smooth objective function (i.e., having $L$-Lipschitz continuous gradients) with a single local step, i.e., $\tau_t^i =1, \forall i \in [m], t \in [T]$.
We note that we consider the original FedAvg where model parameters $\{ \x_t^i, i \in [m] \}$ are aggregated over-the-air without any further scaling.
Consequently, the channel output could be simplified as $\x_{t+1} = \x_t - \eta \nabla F(\x_t, \xi_t) + \w_t$, where $\xi_t \triangleq \{\xi^i_{t}, \forall i \in [m] \}$ represents one collective data batch composed of random samples $\{\xi^i_{t}, \forall i\}$ from all clients.
Then, we have the following theorem to characterize the impact of the channel noise on the OTA version of the FedAvg algorithm:

\begin{restatable}[Lower Bound for Gaussian Channel] {theorem} {lb} \label{thm:lb}
    Consider an OTA-FL system for training an $L$-smooth objective function $F(\x)$ with an optimal solution $\x^*$.
    Supposed that each client uses local SGD updates that are subject to additive white Gaussian noise (AWGN), i.e., $\x_{t+1} = \x_t - \eta \nabla F(\x_t, \xi_t) + \w_t$, where $\eta < \frac{1}{L}$ and $\w_t \sim \mc{N}(\0, \sigma_c^2 \mf{I}_d)$.
    Then, the sequence $\{ \x_t \}$ satisfies the following recursive relationship:
    \begin{align*}
        &\mb{E}\left[ \| \x_{t+1} - \x^* \|^2\right]\geq \mb{E}\left[ \left( 1 - \eta L \right)^2 \| \x_t - \x^* \|^2 \right] + \eta^2 \sigma^2 + \sigma_c^2, 
    \end{align*}
    which further implies the following lower bound for the training convergence: 
    \begin{align}
        &\lim_{t \rightarrow \infty} \mb{E}\left[ \| \x_t - \x^* \|^2\right] \geq \frac{\eta^2 \sigma^2 + \sigma_c^2}{1 - \left(1 - \eta L\right)^2},
    \end{align}
    where the stochastic gradient noise is assumed to be Gaussian, i.e., $\nabla F(\x_t, \xi_t) - \nabla F(\x_t) \sim \mc{N}(\0, \sigma^2 \mf{I}_d)$.
\end{restatable}

\begin{proof}[Proof Sketch]
By assuming independent stochastic gradient noise and channel noise, we could decouple these noise terms and thus producing an iteration relation of $\| \x_t - \x_{*} \|$ by $L$-smoothness with proper learning rate $\eta < \frac{1}{L}$.
As the channel noise exists in every round, such noise variance term $\sigma_c^2$ is non-vanishing even for infinitely many rounds.
Due to space limitation, we refer readers to Appendix~\ref{FullProof} for the complete proof.
\end{proof}

Theorem~\ref{thm:lb} suggests that there is a non-vanishing convergence error due to the Gaussian MAC noise when only standard SGD updates are used locally at each client.
This motivates us to develop joint adaptive computation and power control for OTA-FL to mitigate the MAC noise effect.

\subsection{Impacts of System-Data Heterogeneity on OTA-FL} \label{subsec:heterogeneity}
As discussed in above Section~\ref{subsec: fl}, the FL optimization problem considered in this paper contains non-convex objective function, heterogeneous (non-i.i.d.) data, and different number of local updates $\tau_t^i$ at each client.
As shown in previous work \cite{wang2020fednova}, different number of local steps (or optimization processes) among clients introduce objective inconsistency, rendering potentially arbitrary deviation from optimal solutions in conventional FL.
Next, we show that similar negative impacts of data heterogeneity and different number of local steps also affect the OTA-FL performance even under proper power control.
This insight further motivates the need for a {\em joint} adaptive computation and power control.

Here, we use the GD method for each client's local update and omit the channel noise for now to clearly characterize the above factors in FL. 
That is, each client $i \in [m]$ takes local steps as follows:
\begin{align}
    \x^i_{t, k+1} = \x^i_{t, k} - \eta \nabla F_i(\x^i_{t, k}), \quad \forall k \in [\tau_i].
\end{align}
Then, a power control factor at client $i$ is applied as follows: 
$\z^i_t = \beta_i \left(\x^i_{t, \tau_i} - \x^i_{t, 0}\right)$.
The aggregation with power control factor $\beta$ on the server side can be written as
\begin{align}
    \x_{t+1} - \x_{t} &= \sum_{i=1}^{m} \frac{\beta_i}{\beta} \left(\x^i_{t, \tau_i} - \x^i_{t, 0}\right).
\end{align}
We now consider the following OTA-FL example:
\begin{ex}[Deviation of Objective Value under Disjoint Power Control and System-Data Heterogeneity]
\label{ex:objective_Deviation}
    {\em Consider quadratic objective functions:
    \begin{align*}
        F_i(\x) &= \frac{1}{2} \x_t \MH_i \x - \e_i^T \x + \frac{1}{2} \e_i^T \MH_i^{-1} \e_i, \forall i \in [m], \text{ and }  \\
        F(\x) &= \sum_{i=1}^{m} \alpha_i F_i(\x) = \frac{1}{2} \x_t \bar{\MH} \x - \bar{\e_i}^T \x + \frac{1}{2} \sum_{i=1}^{m} \alpha_i \e_i^T \MH_i^{-1} \e_i,
    \end{align*}
    where $\MH_i \in \mb{R}^{d \times d}$ is invertible, $\bar{\MH} \triangleq \sum_{i=1}^{m} \alpha_i \MH_i$, $\e_i \in \mb{R}^{d}$ is an arbitrary vector, and $\bar{\e} \triangleq \sum_{i=1}^{m} \alpha_i \e_i$.
    Let each client take $\tau_i$ GD steps.
    Then, the generated sequence $\{\x_t\}$ satisfies:
    \begin{align}
        \lim_{t \rightarrow \infty} \x_t &= \hat{\x},
    \end{align}
    where the limit point $\hat{\x}$ can be computed as:
    \begin{eqnarray}
        \hat{\x} &= \left[\sum_{i=1}^{m} \frac{\beta_i}{\beta} \left[\mf{I} - \left(\mf{I} - \eta \MH_i \right)^{\tau_i} \right] \MH_i^{-1} \MH_i\right]^{-1} \times \nonumber\\
        & \left[\sum_{i=1}^{m} \frac{\beta_i}{\beta} \left[\mf{I} - \left(\mf{I} - \eta \MH_i \right)^{\tau_i} \right] \MH_i^{-1} \e_i\right].
    \end{eqnarray}
For the quadratic objective function $F(\x)$, the closed-form solution is $\x^* = \bar{\MH}^{-1} \bar{\e}$.
Comparing $\x^*$ and $\hat{\x}$, we can see a deviation that depends on problem hyper-parameter $\MH$, learning rate $\eta$, local step numbers $\tau_i$, factor $\beta_i$ and $\beta$.}\qed
\end{ex}

Due to space limitation, we provide the proof details of this example in Appendix~\ref{FullProof}.
It can be seen from this example that the complex coupling between power control and system-data heterogeneity renders a highly non-trivial OTA-FL power control and algorithmic design to guarantee convergence to an optimal solution.
This further motivates our OTA-FL algorithm design with {\em joint} adaptive computation and power control in Section~\ref{sec: alg}.


\section{Algorithm Design} \label{sec: alg}

\begin{algorithm}[t!]
    \caption{Adaptive Over-the-Air Federated Learning.} \label{alg:adap_fl} 
    \begin{algorithmic}[1]
    \STATE 
    \emph{\bf Init: global mode $\x_0$.}
    \FOR{$t=0, \dots, T-1$}
    \STATE {Server broadcasts latest global model $\x_t$ to each device.}
    \FOR{each client $i \in [m]$}
        \STATE {Each client locally and adaptively trains model via SGD~\eqref{sgd} under power constraints and transmits $\delta^i_t$ by transmission scaling as in~\eqref{precoding}.} 
    \ENDFOR
    \STATE {The server aggregates and updates global model by receiver rescaling~\eqref{decoding}.}
    \ENDFOR
    \end{algorithmic}
    \end{algorithm}

To address the negative impacts of channel noise and system-data heterogeneity in Section~\ref{sec:channelnoise}, we propose an over-the-air federated learning algorithm with joint adaptive computation and power control (ACPC-OTA-FL) as shown in Algorithm~\ref{alg:adap_fl}.
The basic idea of our {\alg} algorithm is to utilize a time-varying dynamic number of local SGD steps at each client under the instantaneous power constraint at this particular client.
Specifically, given a server-side power control scaling factor $\beta_t$, 
each client $i$ chooses to perform $\tau_t^i$ local SGD steps as follows:
\begin{align}
    \x^i_{t, k+1} = \x^i_{t, k} - \eta \nabla F(\x^i_{t, k}, \xi^i_{t, k}), \quad k=1,\ldots,\tau_t^i, \label{sgd}
\end{align}
under the time-varying and client-dependent power constraints $P_t^i$ such that 
$\| \delta^i_t \|^2 \leq P_t^i$,
where $\delta^i_t = \beta_t^i (\x^i_{t, \tau_t^i} - \x^i_{t, 0})$ and \begin{align} \label{precoding}
    \beta_t^i = \frac{\beta_t \alpha_i}{\tau^i_t}.
\end{align}
Given power constraints $P$ and $\beta$, we can choose $\tau^i_t$ based on {\em trail-and-error} or {\em greedy approach} whenever the power constraints are satisfied.
For simplicity, we ignore fading for now.
The uplink channel output is $\sum_{i=1}^{m} \delta^i_t + \w_t$.
With power control rescaling on the server side, the global model is aggregated and updated as: 
\begin{align}
    \x_{t+1} &= \x_{t} + \sum_{i=1}^{m} \frac{\delta^i_t}{\beta_t} + \tilde{\w}_t. \label{decoding}
\end{align}
As mentioned earlier, we assume a Gaussian channel noise $\w_t \sim \mc{N}(\0, \sigma_c^2 \mf{I}_d)$ and thus $\tilde{\w}_t \sim \mc{N}(\0, \frac{\sigma_c^2}{\beta_t^2} \mf{I}_d)$.
The novelty of our algorithm lies in utilizing a time-varying dynamic local steps to fully exploit the computation and communication power resources.

There are two advantages in our algorithm compared to previous works.
First, we jointly consider the computation-communication co-design due to their complex coupling relationship as shown in Section~\ref{subsec:heterogeneity}.
As a result, more powerful clients with more computation capacities and transmission power will execute more local update steps and have a large fraction in the server-side aggregation.
This adaptive and client-dependent design is different from previous works ~\cite{chen2018uniform,amiri2020machine,amiri2020federated,zhu2020one,guo2020analog,zhang2020gradient,sery2021over}, which considered the communication problem separately after finishing local update computation and used an uniform power control scaling factor without considering the heterogeneity among the clients.
Second, our ACPC-OTA-FL algorithm alleviates the straggler (i.e., slow client) problem by allowing different local step numbers across clients in each communication round. 
%
%
Before providing the theoretical convergence result, we first state our assumptions:
\begin{assum}($L$-Lipschitz Continuous Gradient) \label{a_smooth}
	There exists a constant $L > 0$, such that $ \| \nabla F_i(\x) - \nabla F_i(\y) \| \leq L \| \x - \y \|$, $\forall \x, \y \in \mathbb{R}^d$, and $i \in [m]$.
\end{assum}

\begin{assum}(Unbiased Local Stochastic Gradients and Their Bounded Variance) \label{a_unbias}
	Let $\xi_i$ be a random local data sample at client $i$.
	The local stochastic gradient is unbiased and has a bounded variance, i.e.,
	$\mathbb{E} [\nabla F_i(\x, \xi_i)] = \nabla F_i(\x)$, $\forall i \in [m]$, and $\mathbb{E} [\| \nabla F_i(\x, \xi_i) -  \nabla F_i(\x) \|^2] \leq \sigma^2$, where the expectation is taken over the local data distribution $\mc{X}_i$.
\end{assum}

\begin{assum}(Bounded Stochastic Gradient) \label{a_bounded}
	There exist a constant $G \geq 0$, such that the norm of each local stochastic gradient is bounded:
	$\mathbb{E} [\| \nabla F_i(\x, \xi_i) \|^2] \leq G^2$, $\forall i \in [m]$.
\end{assum}

Assumptions~\ref{a_smooth}--\ref{a_unbias} are common in the analysis of SGD-based algorithms.
Assumption~\ref{a_bounded} is also widely used in OTA-FL with non-i.i.d. datasets (e.g., \cite{amiri2020machine,Li2020convergence,sery2021over}).
With these three assumptions, we have the following convergence result:

\begin{restatable}[Convergence Rate] {theorem} {convergence} \label{thm:convergence}
    Let $\{ \x_t \}$ be the global model generated by Algorithm~\ref{alg:adap_fl}.
    Under Assumptions~\ref{a_smooth}-~\ref{a_bounded} and a constant learning rate $\eta_t = \eta, \forall t \in [T]$, it holds that:
    \begin{multline}
        \min_{t \in [T]} \mb{E} \| \nabla F(\x_t) \|^2 \leq \underbrace{\frac{2 \left(F(\x_0) - F(\x_{*}) \right)}{T \eta}}_{\mathrm{optimization \, error}} + \underbrace{ L \eta \sigma^2 \sum_{i=1}^{m} \alpha_i^2}_{\mathrm{statistical \, error}} \nonumber \\
        + \underbrace{m L^2 \eta^2 G^2 \sum_{i=1}^m (\alpha_i)^2 \left(\tau_i\right)^2}_{\mathrm{local \, update \, error}} + \underbrace{\frac{L \sigma_c^2}{ \eta \beta^2}}_{\substack{\mathrm{channel \,
        noise} \\ \mathrm{error}}},
    \end{multline}
    where $\left(\tau_i\right)^2 = \frac{\sum_{t=0}^{T-1} \left( \tau^i_t \right)^2}{T}$ and $\frac{1}{\bar{\beta}^2} = \frac{1}{T} \sum_{t=0}^{T-1} \frac{1}{\beta_t^2}$.
\end{restatable}

\begin{proof}[Proof Sketch]
In each round, the average global model update is the same as noise-free FedAvg since the channel noise is independent Gaussian with zero mean.
Thus, we can decouple the channel noise term as an extra error scaled by $\frac{\sigma_c^2}{\beta_t^2}$ when calculating the function descent ($F(\x_{t+1}) - F(\x_t)$) in each round by the $L$-smoothness.
Then, the technical challenge lies in heterogeneous local steps across clients.
By simulating mini-batch SGD method, we could further bound the difference $\frac{1}{2} \eta_t \mb{E}_t [\| \sum_{i=1}^{m} \frac{\alpha_i}{\tau^i_t} \sum_{k=0}^{\tau^i_t-1} (\nabla F_i(\x_t) - \nabla F_i(\x^i_{t, k})) \|^2] \leq \frac{1}{2} \eta_t^3 m L^2 \sum_{i=1}^m (\alpha_i)^2\left(\tau^i_t\right)^2 G^2$, which accounts for the size of dataset, data heterogeneity and different local steps.
The above two terms correspond to channel noise error and local update error, respectively.
Following the classic analysis for SGD-based methods, the optimization error and statistical error could be similarly derived, and the final convergence result naturally follows.
Due to space limitation, we relegate the full proof to Appendix~\ref{FullProof}.
\end{proof}

Theorem~\ref{thm:convergence} characterizes four sources of errors that affect the convergence rate: 
1) the optimization error dependent on the distance between the initial guess and optimal objective value; 
2) the statistical error due to the use of stochastic gradients rather than full gradients; 
3) local update error from local update steps coupled with data heterogeneity;  and 
4) channel noise error from over-the-air transmissions. 
Among these four errors, only the optimization error (first term) vanishes as the total number of iterations $T$ gets large, while other three terms are independent of $T$.

Similar to classic SGD or FedAvg convergence analysis, diminishing learning rates $\mc{O}(\frac{1}{\sqrt{T}})$ can be used to remove the statistical and local update errors and obtain a convergence error bound $\min_{t \in [T]} \mb{E} \| \nabla F(\x_t) \|^2 = \mc{O}(\frac{1}{\sqrt{T}})$.
To mitigate the channel noise error, the parameter $\beta$ needs to be chosen judiciously.
Given $\delta_t^i$ in communication round $t$, we can set $\beta_t^2 = \min_{i \in [m]} \{\frac{P_t^i (\tau_t^i)^2}{\| \delta_t^i \|^2 \alpha_i^2}\}$.
If the $\delta_t^i$-information is unavailable, we can choose $\| \delta_t^i \|^2 \leq \eta_t^2 (\tau_t^i)^2 G^2$ by its upper bound, and thus $\beta_t^2 = \frac{P_t}{\alpha_i^2 \eta_t^2 G^2}$, where $P_t = \min_{i \in [m]} P_t^i$.
For the special case with $P = P_t^i, \forall i, t$, $\alpha_i = \frac{1}{m}$ (balanced datasets), and identical local steps $\tau_t^i = \tau, \forall i, t$, the channel noise error (the fourth term) becomes $\frac{\eta \sigma_c^2 G^2}{P m^2}$, 
and the following result immediately follows from Theorem~\ref{thm:convergence}:

\begin{restatable}[Convergence Rate] {corollary} {convergence_rate} \label{cor:convergence}
Let $\alpha_i = \frac{1}{m}, \tau_i = \tau, \eta = \frac{\sqrt{m}}{\sqrt{T}}, \beta^2 = \frac{m}{\eta^2}$, the convergence rate of {\alg} under the special case above is $\mc{O}(\frac{\sigma^2 + 1}{\sqrt{mT}}) + \mc{O}(\frac{m \tau^2 G^2}{T}) + \mc{O}(\frac{ \sigma_c^2}{\sqrt{mT}})$.
\end{restatable}

Corollary~\ref{cor:convergence} implies that, if $\tau \leq \frac{T^{1/4}}{m^{3/4}}$, a linear speedup in terms of the number of clients (i.e., $\mc{O}(\frac{1}{\sqrt{mT}})$) can be achieved, which shows the benefits of parallelism and matches the convergence rate of FedAvg in noise-free communication environment~\cite{yu2019parallel,yang2021linearspeedup}.


Lastly, we note that it is straightforward to extend our results to fading channels with known CSI.
Specifically, the adaptive computation and power control strategy for fading channels is to choose local steps $\tau^i_t$ such that $\| \delta^i_t \|^2 \leq P_t^i$, where $\delta^i_t = \beta_t^i \left(\x^i_{t, \tau_i} - \x^i_{t, 0}\right),
\beta_t^i = \frac{\beta_t \alpha_i}{\tau^i_t h_t^i}$, and $P_t^i > 0$ represents the maximum transmission power for client $i$ in round $t$.
Under this joint computation and power control, the received signal remains the same as that in the non-fading OTA-FL setting.
Thus, the same convergence results in Theorem~\ref{thm:convergence} and Corollary~\ref{cor:convergence} continue to hold.


\section{Numerical Results} \label{sec: exp}

In this section, we conduct numerical experiments to verify our theoretical results using logistic regression on the MNIST dataset~\cite{lecun1998gradient}. Following the same procedure as in existing works~\cite{Li2020convergence, mcmahan2017communication, yang2021linearspeedup}, we distribute the data evenly to $m=10$ clients in a label-based partition to impose data heterogeneity across the clients, where the heterogeneity level can be characterized by a parameter $p$. 
As the MNIST dataset contains $10$ classes of labels in total, $p=10$ represents the i.i.d. case.
The smaller the $p$-value, the more heterogeneous the data across clients. 
We simulate Gaussian MAC with signal-to-noise ratios (SNRs) of \unit[$-1$]{dB}, \unit[10]{dB} and \unit[20]{dB}.

\begin{table}[tb]
\caption{Logistic regression test Accuracy (\%) for ACPC-OTA-FL compared with COTAF and FedAvg on the MNIST dataset.}
\label{table:test_accuracy}
\centering
\begin{tabular}{|c|c|ccc|}
\hline
\multicolumn{1}{|c|}{\multirow{2}{*}{\textbf{Non-IID Level}}} & \multicolumn{1}{c|}{\multirow{2}{*}{\textbf{Algorithm}}} & \multicolumn{3}{c|}{\textbf{Signal-to-Noise Ratio}}                                        \\ \cline{3-5} 
\multicolumn{1}{|c|}{}                               & \multicolumn{1}{c|}{}                           & \multicolumn{1}{c|}{\unit[-1]{dB}}  & \multicolumn{1}{c|}{\unit[10]{dB}}  & \unit[20]{dB}  \\ \hline
\multirow{3}{*}{$p=1$}                                 & {\cellcolor[gray]{.9}{\bf ACPC-OTA-FL}}                                     & \multicolumn{1}{c|}{\cellcolor[gray]{.9} {\bf 78.22}} & \multicolumn{1}{c|}{\cellcolor[gray]{.9} {\bf 89.08}} & \cellcolor[gray]{.9} {\bf 89.54} \\ \cline{2-5} 
                                                     & COTAF                                           & \multicolumn{1}{c|}{46.55} & \multicolumn{1}{c|}{65.54} & 85.92 \\ \cline{2-5} 
                                                     & FedAvg                                          & \multicolumn{1}{c|}{67.49} & \multicolumn{1}{c|}{68.08} & 67.65 \\ \hline
\multirow{3}{*}{$p=2$}                                 & {\cellcolor[gray]{.9}{\bf ACPC-OTA-FL}}                                     & \multicolumn{1}{c|}{\cellcolor[gray]{.9} {\bf 81.89}} & \multicolumn{1}{c|}{\cellcolor[gray]{.9} {\bf 89.58}} & \cellcolor[gray]{.9} {\bf 90.43} \\ \cline{2-5} 
                                                     & COTAF                                           & \multicolumn{1}{c|}{63.59} & \multicolumn{1}{c|}{78.80} & 86.57 \\ \cline{2-5} 
                                                     & FedAvg                                          & \multicolumn{1}{c|}{71.55} & \multicolumn{1}{c|}{78.03} & 79.86 \\ \hline
\multirow{3}{*}{$p=5$}                                 & {\cellcolor[gray]{.9}{\bf ACPC-OTA-FL}}                                     & \multicolumn{1}{c|}{\cellcolor[gray]{.9} {\bf 86.48}} & \multicolumn{1}{c|}{\cellcolor[gray]{.9} {\bf 90.64}} & \cellcolor[gray]{.9}{\bf 91.20} \\ \cline{2-5} 
                                                     & COTAF                                           & \multicolumn{1}{c|}{79.64} & \multicolumn{1}{c|}{86.52} & 90.82 \\ \cline{2-5} 
                                                     & FedAvg                                          & \multicolumn{1}{c|}{74.76} & \multicolumn{1}{c|}{82.84} & 85.40 \\ \hline
\multirow{3}{*}{$p=10$}                                & {\cellcolor[gray]{.9}{\bf ACPC-OTA-FL}}                                     & \multicolumn{1}{c|}{\cellcolor[gray]{.9}{\bf86.21}} & \multicolumn{1}{c|}{\cellcolor[gray]{.9}{\bf 90.75}} & \cellcolor[gray]{.9} {\bf 91.08} \\ \cline{2-5} 
                                                     & COTAF                                           & \multicolumn{1}{c|}{86.43} & \multicolumn{1}{c|}{91.08} & 92.63 \\ \cline{2-5} 
                                                     & FedAvg                                          & \multicolumn{1}{c|}{76.26} & \multicolumn{1}{c|}{84.94} & 88.11 \\ \hline
\end{tabular}
\vspace{-.1in}
\end{table}

We illustrate the test accuracy of {\alg} compared with COTAF~\cite{sery2021over} and FedAvg~\cite{mcmahan2017communication} in Table~\ref{table:test_accuracy}. 
Two key observations are in order: 
1) Test accuracy drops significantly by directly applying FedAvg algorithm to wireless OTA-FL (up to $20\%$ accuracy drop) under large channel noise, which validates our Theorem~\ref{thm:lb} and is consistent with existing works\cite{zhu2019broadband,sery2020analog,amiri2020federated};
and 2) Under power control, both {\alg} and COTAF could alleviate the channel noise impacts in the i.i.d. data setting ($p=10$).
But in the highly heterogeneous data ($p=1, 2$) and/or low SNR settings, our {\alg} algorithm outperforms COTAF by a large margin.
For example, when $\text{SNR}=\unit[-1]{dB}$ and $p=1$, {\alg} improves the test accuracy by 31.76\% and 10.73\% compared to COTAF and FedAvg, respectively.
The intuition is that the gradient returned from the clients vary dramatically in highly heterogeneous data settings, and thus utilizing an adaptive local steps under limited power constraints allows each client to fully exploit both computation and communication resources.


\section{Conclusion} \label{sec: conclusion}
In this paper, we considered the joint adaptive local computation (number of local steps) and power control for OTA-FL. 
We first characterized the training error due to channel noise for conventional OTA-FL by establishing a fundamental lower bound for general objective functions with Lipschitz-continuous gradients. 
This motivated us to propose an over-the-air federated learning algorithm with joint adaptive computation and power control ({\alg}) to mitigate the impacts of channel noise on the learning performance, while taking the device heterogeneity into consideration. 
We analyzed the convergence of {\alg} with non-convex objective functions and heterogeneous data, and shown that the convergence rate of {\alg} matches that of FedAvg with noise-free communications.






\onecolumn
\allowdisplaybreaks

\section{Proof}\label{FullProof}
\lb*

\begin{proof}
    \begin{align}
        \mb{E}\left[ \| \x_{t+1} - \x^* \|^2\right] &= \mb{E}\left[ \| \x_t - \eta \nabla F(\x_t, \xi_t) + \w_t - \x^* \|^2\right] \\
        &= \mb{E}\left[ \| \x_t - \x^* - \eta \nabla F(\x_t) \|^2 \right] + \mb{E}\left[ \| \eta \nabla F(\x_t, \xi_t) - \eta \nabla F(\x_t) \|^2 \right] + \mb{E}\left[ \| \w_t  \|^2\right] \\
        &= \mb{E}\left[ \| \x_t - \x^* - \eta \left( \nabla F(\x_t) - \nabla F(\x^*) \right)  \|^2 \right] + \eta^2 \sigma^2 + \sigma_c^2 \\
        &\geq \mb{E}\left[ \left( \| \x_t - \x^* \| - \eta \| \nabla F(\x_t) - \nabla F(\x^*) \| \right)^2 \right] + \eta^2 \sigma^2 + \sigma_c^2 \\
        &\geq \mb{E}\left[ \left( 1 - \eta L \right)^2 \| \x_t - \x^* \|^2 \right] + \eta^2 \sigma^2 + \sigma_c^2
    \end{align}
\end{proof}

\textbf{Example 1}
{\em Consider quadratic objective functions:
    \begin{align*}
        F_i(\x) &= \frac{1}{2} \x_t \MH_i \x - \e_i^T \x + \frac{1}{2} \e_i^T \MH_i^{-1} \e_i, \forall i \in [m], \text{ and }  \\
        F(\x) &= \sum_{i=1}^{m} \alpha_i F_i(\x) = \frac{1}{2} \x_t \bar{\MH} \x - \bar{\e_i}^T \x + \frac{1}{2} \sum_{i=1}^{m} \alpha_i \e_i^T \MH_i^{-1} \e_i,
    \end{align*}
    where $\MH_i \in \mb{R}^{d \times d}$ is invertible, $\bar{\MH} \triangleq \sum_{i=1}^{m} \alpha_i \MH_i$, $\e_i \in \mb{R}^{d}$ is an arbitrary vector, and $\bar{\e} \triangleq \sum_{i=1}^{m} \alpha_i \e_i$.
    Let each client take $\tau_i$ GD steps.
    Then, the generated sequence $\{\x_t\}$ satisfies:
    \begin{align}
        \lim_{t \rightarrow \infty} \x_t &= \hat{\x},
    \end{align}
    where the limit point $\hat{\x}$ can be computed as:
    \begin{eqnarray}
        \hat{\x} &= \left[\sum_{i=1}^{m} \frac{\beta_i}{\beta} \left[\mf{I} - \left(\mf{I} - \eta \MH_i \right)^{\tau_i} \right] \MH_i^{-1} \MH_i\right]^{-1} \times \nonumber\\
        & \left[\sum_{i=1}^{m} \frac{\beta_i}{\beta} \left[\mf{I} - \left(\mf{I} - \eta \MH_i \right)^{\tau_i} \right] \MH_i^{-1} \e_i\right].
    \end{eqnarray}
For the quadratic objective function $F(\x)$, the closed-form solution is $\x^* = \bar{\MH}^{-1} \bar{\e}$.
Comparing $\x^*$ and $\hat{\x}$, we can see a deviation that depends on problem hyper-parameter $\MH$, learning rate $\eta$, local step numbers $\tau_i$, factor $\beta_i$ and $\beta$.}\qed

\begin{proof}
Consider quadratic objective functions:
\begin{align}
    F_i(\x) &= \frac{1}{2} \x_t \MH_i \x - \e_i^T \x + \frac{1}{2} \e_i^T \MH_i^{-1} \e_i, \\
    F(\x) &= \sum_{i=1}^{m} \alpha_i F_i(\x) = \frac{1}{2} \x_t \bar{\MH} \x - \bar{\e_i}^T \x + \frac{1}{2} \sum_{i=1}^{m} \alpha_i \e_i^t \MH_i^{-1} \e_i,
\end{align}
where $\MH_i \in \mb{R}^{d \times d}$ is invertible matrix, $\bar{\MH} = \sum_{i=1}^{m} \alpha_i \MH_i$, $\e_i \in \mb{R}^{d}$ is arbitrary vector, and $\bar{\e} = \sum_{i=1}^{m} \alpha_i \e_i$.
It is easy to show the optimum for each local and global objective function are $\x_i^* = \MH_i^{-1} \e_i, \forall i \in [m]$ and $\x^* = \bar{\MH}^{-1} \bar{\e}$.

The local update for each client $i \in [m]$ is as follows:
\begin{align}
    \x^i_{t, k+1} &= \x^i_{t, k} - \eta \left[\MH_i \x^i_{t, k} - \e_i \right] \\
    &= \left(\mf{I} - \eta \MH_i \right) \x^i_{t, k} + \eta \e_i \label{ine:onestep}
\end{align}

Then we have the recursive equation for one local step by rearranging ~\ref{ine:onestep}: 
\begin{align}
    \x^i_{t, k+1} - \c^i_t &= \left(\mf{I} - \eta \MH_i \right) \left( \x^i_{t, k} - \c^i_t \right),
\end{align}
where $\c^i_t = \MH_i^{-1} \e_i$.

\begin{align}
    \x^i_{t, \tau_i} - \c^i_t &= \left(\mf{I} - \eta \MH_i \right)^{\tau_i} \left( \x^i_{t, 0} - \c^i_t \right), \\
    \x^i_{t, \tau_i} - \x^i_{t, 0} &= \left[ \left(\mf{I} - \eta \MH_i \right)^{\tau_i} - \mf{I} \right] \MH_i^{-1} \left( \e_i - \MH_i \x^i_{t, 0} \right) \\
    &= \K_i(\eta) \left( \e_i - \MH_i \x^i_{t, 0} \right),
\end{align}
where we define $\K_i(\eta) = \left[\mf{I} - \left(\mf{I} - \eta \MH_i \right)^{\tau_i} \right] \MH_i^{-1}$.

Aggregation:
\begin{align}
    \x_{t+1} - \x_{t} &= \sum_{i=1}^{m} \frac{\beta_i}{\beta} \left(\x^i_{t, \tau_i} - \x^i_{t, 0}\right) \\
    &= \sum_{i=1}^{m} \frac{\beta_i}{\beta} \K_i(\eta) \left( \e_i - \MH_i \x^i_{t, 0} \right), \\
    \x_{t+1} &= \left[\mf{I} - \sum_{i=1}^{m} \frac{\beta_i}{\beta} \K_i(\eta) \MH_i \right] \x_{t} + \sum_{i=1}^{m} \frac{\beta_i}{\beta} \K_i(\eta) \e_i, \\
    \x_{t+1} - \hat{\x} &= \left[\mf{I} - \sum_{i=1}^{m} \frac{\beta_i}{\beta} \K_i(\eta) \MH_i \right] \left(\x_{t} - \hat{\x}\right), \\
    \x_{t} &= \left[\mf{I} - \sum_{i=1}^{m} \frac{\beta_i}{\beta} \K_i(\eta) \MH_i \right]^{t} \left(\x_0 - \hat{\x}\right) + \hat{\x},
\end{align}
where $\hat{\x} := \left[\sum_{i=1}^{m} \frac{\beta_i}{\beta} \K_i(\eta) \MH_i\right]^{-1} \left[\sum_{i=1}^{m} \frac{\beta_i}{\beta} \K_i(\eta) \e_i\right] = \left[\sum_{i=1}^{m} \frac{\beta_i}{\beta} \left[\mf{I} - \left(\mf{I} - \eta \MH_i \right)^{\tau_i} \right] \MH_i^{-1} \MH_i\right]^{-1} \left[\sum_{i=1}^{m} \frac{\beta_i}{\beta} \left[\mf{I} - \left(\mf{I} - \eta \MH_i \right)^{\tau_i} \right] \MH_i^{-1} \e_i\right]$.

As $t$ goes to sufficiently large, we have
$$\lim_{t \rightarrow \infty} \x_t = \hat{\x}.$$

\end{proof}

\textbf{Theorem 2} (Convergence Rate).
Let $\{ \x_t \}$ be the global model generated by Algorithm~\ref{alg:adap_fl}.
    Under Assumptions~\ref{a_smooth}-~\ref{a_bounded} and a constant learning rate $\eta_t = \eta, \forall t \in [T]$, it holds that:
    \begin{multline}
        \min_{t \in [T]} \mb{E} \| \nabla F(\x_t) \|^2 \leq \underbrace{\frac{2 \left(F(\x_0) - F(\x_{*}) \right)}{T \eta}}_{\mathrm{optimization \, error}} + \underbrace{ L \eta \sigma^2 \sum_{i=1}^{m} \alpha_i^2}_{\mathrm{statistical \, error}} + \underbrace{m L^2 \eta^2 G^2 \sum_{i=1}^m (\alpha_i)^2 \left(\tau_i\right)^2}_{\mathrm{local \, update \, error}} + \underbrace{\frac{L \sigma_c^2}{ \eta \beta^2}}_{\substack{\mathrm{channel \,
        noise} \\ \mathrm{error}}},
    \end{multline}
    where $\left(\tau_i\right)^2 = \frac{\sum_{t=0}^{T-1} \left( \tau^i_t \right)^2}{T}$ and $\frac{1}{\bar{\beta}^2} = \frac{1}{T} \sum_{t=0}^{T-1} \frac{1}{\beta_t^2}$.

\begin{proof}

\begin{align}
    \x_{t+1} - \x_{t} &= \sum_{i=1}^{m} \frac{\beta_i}{\beta} \left(\x^i_{t, \tau^i_t} - \x^i_{t, 0}\right) + \tilde{\w}_t \\
    &= \sum_{i=1}^{m} \frac{\alpha_i}{\tau^i_t} \left(\x^i_{t, \tau^i_t} - \x^i_{t, 0}\right) + \tilde{\w}_t \\
    &= \sum_{i=1}^{m} \frac{\alpha_i}{\tau^i_t} \eta_t \sum_{k=0}^{\tau^i_t-1} \left(\nabla F_i(\x^i_{t, k}, \xi^i_{t, k})\right) + \tilde{\w}_t
\end{align}

Due to L-smoothness, we have one step descent in expectation conditioned on $\x_t$,
\begin{align}
    &\mb{E}_t [F(\x_{t+1})] - F(\x_t) \leq \left< \nabla F(\x_t), \mb{E}_t \left[\x_{t+1} - \x_t \right] \right> + \frac{L}{2} \mb{E}_t \left[\| \x_{t+1} - \x_t \|^2 \right] \\
    &= \left< \nabla F(\x_t), \sum_{i=1}^{m} \frac{\alpha_i}{\tau^i_t} \eta_t \sum_{k=0}^{\tau^i_t-1} \left(\nabla F_i(\x^i_{t, k})\right) \right> + \frac{L}{2} \mb{E}_t \left[\bigg\| \sum_{i=1}^{m} \frac{\alpha_i \eta_t}{\tau^i_t} \sum_{k=0}^{\tau^i_t-1} \left(\nabla F_i(\x^i_{t, k}, \xi^i_{t, k})\right)  + \tilde{\w}_t \bigg\|^2 \right] \\
    &= - \eta_t \| \nabla F(\x_t) \|^2 + \eta_t \left< \nabla F(\x_t), \nabla F(\x_t) - \sum_{i=1}^{m} \frac{\alpha_i}{\tau^i_t} \sum_{k=0}^{\tau^i_t-1} \left(\nabla F_i(\x^i_{t, k})\right) \right> \\
    &+ \frac{L}{2} \mb{E}_t \left[\bigg\| \sum_{i=1}^{m} \frac{\alpha_i \eta_t}{\tau^i_t} \sum_{k=0}^{\tau^i_t-1} \left(\nabla F_i(\x^i_{t, k}, \xi^i_{t, k})\right) + \tilde{\w}_t \bigg\|^2 \right] \\
    &=- \frac{1}{2} \eta_t \| \nabla F(\x_t) \|^2 - \frac{1}{2} \eta_t \bigg\| \sum_{i=1}^{m} \frac{\alpha_i}{\tau^i_t} \sum_{k=0}^{\tau^i_t-1} \left(\nabla F_i(\x^i_{t, k})\right) \bigg\|^2 + \frac{1}{2} \eta_t \mb{E}_t \bigg\| \nabla F(\x_t) - \sum_{i=1}^{m} \frac{\alpha_i}{\tau^i_t} \sum_{k=0}^{\tau^i_t-1} \left(\nabla F_i(\x^i_{t, k})\right) \bigg\|^2 \\
    &+ \frac{L}{2} \mb{E}_t \left[\bigg\| \sum_{i=1}^{m} \frac{\alpha_i \eta_t}{\tau^i_t} \sum_{k=0}^{\tau^i_t-1} \left(\nabla F_i(\x^i_{t, k}, \xi^i_{t, k})\right) + \tilde{\w}_t \bigg\|^2 \right] \\
    &=- \frac{1}{2} \eta_t \| \nabla F(\x_t) \|^2 - \frac{1}{2} \eta_t \bigg\| \sum_{i=1}^{m} \frac{\alpha_i}{\tau^i_t} \sum_{k=0}^{\tau^i_t-1} \left(\nabla F_i(\x^i_{t, k})\right) \bigg\|^2 + \frac{1}{2} \eta_t \mb{E}_t \bigg\| \sum_{i=1}^{m} \frac{\alpha_i}{\tau^i_t} \sum_{k=0}^{\tau^i_t-1} \left(\nabla F_i(\x_t) - \nabla F_i(\x^i_{t, k})\right) \bigg\|^2 \\
    &+ \frac{L \eta_t^2}{2} \mb{E}_t \left[\bigg\| \sum_{i=1}^{m} \frac{\alpha_i}{\tau^i_t} \sum_{k=0}^{\tau^i_t-1} \left(\nabla F_i(\x^i_{t, k}, \xi^i_{t, k})\right) \bigg\|^2 \right] + \frac{L \sigma_c^2}{2\beta_t^2} \\
    &\leq - \frac{1}{2} \eta_t \| \nabla F(\x_t) \|^2 + \frac{1}{2} \eta_t \mb{E}_t \bigg\| \sum_{i=1}^{m} \frac{\alpha_i}{\tau^i_t} \sum_{k=0}^{\tau^i_t-1} \left(\nabla F_i(\x_t) - \nabla F_i(\x^i_{t, k})\right) \bigg\|^2 \\
    &+ \frac{L \eta_t^2}{2} \mb{E}_t \left[\bigg\| \sum_{i=1}^{m} \frac{\alpha_i}{\tau^i_t} \sum_{k=0}^{\tau^i_t-1} \left(\nabla F_i(\x^i_{t, k}, \xi^i_{t, k})\right) - \sum_{i=1}^{m} \frac{\alpha_i}{\tau^i_t} \sum_{k=0}^{\tau^i_t-1} \left(\nabla F_i(\x^i_{t, k})\right) \bigg\|^2 \right] + \frac{L \sigma_c^2}{2\beta_t^2} \\
    &\leq - \frac{1}{2} \eta_t \| \nabla F(\x_t) \|^2 + \frac{1}{2} \eta_t \mb{E}_t \bigg\| \sum_{i=1}^{m} \frac{\alpha_i}{\tau^i_t} \sum_{k=0}^{\tau^i_t-1} \left(\nabla F_i(\x_t) - \nabla F_i(\x^i_{t, k})\right) \bigg\|^2 \\
    &+ \frac{L \eta_t^2}{2} \sum_{i=1}^{m} \mb{E}_t \left[\bigg\| \frac{\alpha_i}{\tau^i_t} \sum_{k=0}^{\tau^i_t-1} \left(\nabla F_i(\x^i_{t, k}, \xi^i_{t, k}) - \nabla F_i(\x^i_{t, k})\right) \bigg\|^2 \right] + \frac{L \sigma_c^2}{2\beta_t^2} \label{ineq1} \\
    &\leq - \frac{1}{2} \eta_t \| \nabla F(\x_t) \|^2 + \frac{1}{2} \eta_t \mb{E}_t \bigg\| \sum_{i=1}^{m} \frac{\alpha_i}{\tau^i_t} \sum_{k=0}^{\tau^i_t-1} \left(\nabla F_i(\x_t) - \nabla F_i(\x^i_{t, k})\right) \bigg\|^2 + \frac{L \eta_t^2}{2} \sum_{i=1}^{m} \alpha_i^2 \sigma^2 + \frac{L \sigma_c^2}{2\beta_t^2} \label{ineq: smooth}
\end{align}

The first inequality holds if $\eta_t \leq \frac{1}{L}$.

\begin{align}
    &\frac{1}{2} \eta_t \mb{E}_t \bigg\| \sum_{i=1}^{m} \frac{\alpha_i}{\tau^i_t} \sum_{k=0}^{\tau^i_t-1} \left(\nabla F_i(\x_t) - \nabla F_i(\x^i_{t, k})\right) \bigg\|^2 \leq
    \frac{1}{2} \eta_t m \sum_{i=1}^m \frac{(\alpha_i)^2}{(\tau^i_t)^2} \mb{E}_t \bigg\|  \sum_{k=0}^{\tau^i_t-1} \left(\nabla F_i(\x_t) - \nabla F_i(\x^i_{t, k})\right) \bigg\|^2 \\
    &\leq \frac{1}{2} \eta_t m \sum_{i=1}^m \frac{(\alpha_i)^2}{\tau^i_t} \sum_{k=0}^{\tau^i_t-1} \mb{E}_t \bigg\| \left(\nabla F_i(\x_t) - \nabla F_i(\x^i_{t, k})\right) \bigg\|^2 \\
    &\leq \frac{1}{2} \eta_t m L^2 \sum_{i=1}^m \frac{(\alpha_i)^2}{\tau^i_t} \sum_{k=0}^{\tau^i_t-1} \mb{E}_t \bigg\| \left(\x_t - \x^i_{t, k} \right) \bigg\|^2 \\
    &\leq \frac{1}{2} \eta_t^3 m L^2 \sum_{i=1}^m \frac{(\alpha_i)^2}{\tau^i_t} \sum_{k=0}^{\tau^i_t-1} \mb{E}_t \bigg\| \sum_{j=0}^k \nabla F_i(\x^i_{t, j}, \xi^i_{t, j}) \bigg\|^2 \\
    &\leq \frac{1}{2} \eta_t^3 m L^2 \sum_{i=1}^m \frac{(\alpha_i)^2}{\tau^i_t} \sum_{k=0}^{\tau^i_t-1} k^2 G^2 \\
    &\leq \frac{1}{2} \eta_t^3 m L^2 \sum_{i=1}^m (\alpha_i)^2\left(\tau^i_t\right)^2 G^2 \label{ineq: variance}
\end{align}

Plugging inequality~\eqref{ineq: variance} into \eqref{ineq: smooth}, we have 
\begin{align}
    &\mb{E}_t [F(\x_{t+1})] - F(\x_t) \leq \left< \nabla F(\x_t), \mb{E}_t \left[\x_{t+1} - \x_t \right] \right> + \frac{L}{2} \mb{E}_t \left[\| \x_{t+1} - \x_t \|^2 \right] \\
    &\leq - \frac{1}{2} \eta_t \| \nabla F(\x_t) \|^2 + \frac{1}{2} \eta_t^3 m L^2 \sum_{i=1}^m (\alpha_i)^2\left(\tau^i_t\right)^2 G^2 + \frac{L \eta_t^2}{2} \sum_{i=1}^{m} \alpha_i^2 \sigma^2 + \frac{L \sigma_c^2}{2\beta_t^2}
\end{align}

Rearranging and telescoping:
\begin{align}
    \frac{1}{T} \sum_{t=0}^{T-1} \eta_t \mb{E}_t \| \nabla F(\x_t) \|^2 &\leq \frac{2 \left(F(\x_0) - F(\x_T) \right)}{T} +  m L^2 \frac{1}{T} \sum_{t=0}^{T-1} \eta_t^3 \sum_{i=1}^m (\alpha_i)^2 \left( \tau^i_t \right)^2 G^2 + L \sigma^2 \sum_{i=1}^{m} \alpha_i^2 \frac{1}{T} \sum_{t=0}^{T-1} \eta_t^2 + L \sigma_c^2 \frac{1}{T} \sum_{t=0}^{T-1} \frac{1}{\beta_t^2}
\end{align}

Let $\eta_t = \eta$ be constant learning rate, $\left(\tau_i\right)^2 = \frac{\sum_{t=0}^{T-1} \left( \tau^i_t \right)^2}{T}$ then we have:
\begin{align}
    \frac{1}{T} \sum_{t=0}^{T-1} \mb{E} \| \nabla F(\x_t) \|^2 &\leq \frac{2 \left(F(\x_0) - F(\x_T) \right)}{T \eta} + m L^2 \eta^2 G^2 \sum_{i=1}^m (\alpha_i)^2 \left(\tau_i\right)^2 + L \eta \sigma^2 \sum_{i=1}^{m} \alpha_i^2  + \frac{L \sigma_c^2}{\eta \beta^2},
\end{align}
where $\frac{1}{\bar{\beta}^2} = \frac{1}{T} \sum_{t=0}^{T-1} \frac{1}{\beta_t^2}$.
\end{proof}

\end{document}